\documentclass[letterpaper, 10 pt, conference]{ieeeconf}  % Comment this line out if you need a4paper
\IEEEoverridecommandlockouts                              
\usepackage{graphicx}
\usepackage{amsmath} % assumes amsmath package installed
\usepackage{amssymb}  % assumes amsmath package installed
\usepackage{algorithm,algorithmicx,algpseudocode,booktabs,xcolor}
\usepackage[labelformat=simple]{subcaption}
\usepackage[hidelinks]{hyperref}
\usepackage[export]{adjustbox}

\newcommand{\mb}[1]{\boldsymbol{#1}}
\newcommand\trsp{{\!\scriptscriptstyle\top}}
\newcommand{\etal}{\textit{et al. }}

\newcommand{\bluetext}[1]{#1}
\DeclareMathOperator*{\argmin}{arg\,min}
\newlength{\Oldarrayrulewidth}
\newcommand{\Cline}[2]{%
	\noalign{\global\setlength{\Oldarrayrulewidth}{\arrayrulewidth}}%
	\noalign{\global\setlength{\arrayrulewidth}{#1}}\cline{#2}%
	\noalign{\global\setlength{\arrayrulewidth}{\Oldarrayrulewidth}}}

\newtheorem{theorem}{Theorem}
\newtheorem{lemma}{Lemma}

\title{\bf Uncertainty-Aware Imitation Learning\\using Kernelized Movement Primitives}

\author{Jo\~ao Silv\'erio$^{1,2}$, Yanlong Huang$^2$, Fares J. Abu-Dakka$^{2,3}$, Leonel Rozo$^4$ and Darwin G. Caldwell$^2$
\thanks{$^1$Idiap Research Institute, CH-1920 Martigny, Switzerland (e-mail: joao.silverio@idiap.ch).}
\thanks{$^2$Department of Advanced Robotics, Istituto Italiano di Tecnologia, 16163 Genova, Italy (e-mail: name.surname@iit.it)}
\thanks{$^3$Intelligent Robotics Group, EEA, Aalto University, FI-00076 Aalto, Finland (e-mail: fares.abu-dakka@aalto.fi).}
\thanks{$^4$Bosch Center for Artificial Intelligence, 71272 Renningen, Germany (e-mail:leonel.rozo@de.bosch.com).}
\thanks{Jo\~ao Silv\'erio is partially supported by the CoLLaboratE project (https://collaborate-project.eu/), funded by the EU within H2020-DT-FOF-02-2018 under grant agreement 820767. Fares J. Abu-Dakka is partially supported by CHIST-ERA project IPALM (Academy of Finland decision 326304)}}

\begin{document}

\maketitle
\thispagestyle{empty}
\pagestyle{empty}

\begin{abstract}

During the past few years, probabilistic approaches to imitation learning have earned a relevant place in the robotics literature. One of their most prominent features is that, in addition to extracting a mean trajectory from task demonstrations, they provide a variance estimation. The intuitive meaning of this variance, however, changes across different techniques, indicating either \textit{variability} or \textit{uncertainty}. In this paper we leverage kernelized movement primitives (KMP) to provide a new perspective on imitation learning by predicting variability, correlations and uncertainty using a single model. This rich set of information is used in combination with the \textit{fusion of optimal controllers} to learn robot actions from data, with two main advantages: i) robots become safe when uncertain about their actions and ii) they are able to leverage partial demonstrations, given as elementary sub-tasks, to optimally perform a higher level, more complex task. We showcase our approach in a painting task, where a human user and a KUKA robot collaborate to paint a wooden board. The task is divided into two sub-tasks and we show that the robot becomes compliant (hence safe) outside the training regions and executes the two sub-tasks with optimal gains otherwise.

\end{abstract}

\section{Introduction}
Probabilistic approaches to imitation learning  \cite{Argall09} have witnessed a rise in popularity during the past few years. They are often seen as complementing deterministic techniques, such as dynamic movement primitives \cite{Ijspeert13}, with more complete descriptions of demonstration data, in particular in the form of covariance matrices that encode both the variability and correlations in the data. Widely used approaches at this level include Gaussian mixture models (GMM), popularized by the works of Calinon (e.g. \cite{Calinon14ICRA}) and more recently, probabilistic movement primitives \cite{Paraschos13} and kernelized movement primitives (KMP) \cite{Huang2019d}.
%The aforementioned techniques have in common the fact that they retrieve a mean trajectory in the data, that can serve as a reference policy for a robot controller, and the variability and correlations at new inputs. The latter have been used, for example, in conjunction with various control frameworks to modulate control gains in inverse proportion to the variability in the data \cite{Calinon14ICRA,Medina2012}.
%This draws from the assumption that regions in the training dataset with higher variability allow for greater flexibility in the skill reproduction and need not be as accurately replicated.
%Other approaches have leveraged this information to have the robot autonomously select the most relevant controllers to reproduce a task at every moment \cite{Calinon09AR, Schneider2010,Silverio2018b}. %The two problems, however, have traditionally been treated separately.

In recent work \cite{Silverio2018b, Silverio2018a}, we discussed a fundamental difference between the type of variance encapsulated by the predictions of classical probabilistic techniques, particularly Gaussian mixture regression (GMR) and Gaussian process regression (GPR) \cite{Rasmussen06}. We showed that the variance predicted by these two techniques has distinct, complementary interpretations. In particular that GMR predictions measure the \textit{variability} in the training data, while those of GPR quantify the degree of \textit{uncertainty}, increasing as one queries a model farther away from the region where it was trained. These properties are illustrated in Fig. \ref{fig:GMRvsGPR}. This finding led us to inquire: is there a probabilistic technique that can simultaneously predict \emph{both variability and uncertainty}? Are these two notions compatible and unifiable into a single imitation learning framework where they both provide clear advantages from a learning perspective? In this paper we try to answer these questions.

The two types of variance have been individually leveraged by different lines of work. For instance, variability and data correlations (encapsulated in full covariance matrices) have been used to modulate control gains in several works \cite{Calinon14ICRA,Medina2012,Lee2015,Huang2018}.
%Other approaches have leveraged this information to have the robot autonomously select the most relevant controllers to reproduce a task at every moment \cite{Calinon09AR, Schneider2010,Silverio2018b}.
Uncertainty, in the sense of absence of data/information, is also a concept with tradition in robotics. Problems in robot localization \cite{Fox1999}, control \cite{Medina15} and, more recently, Bayesian optimization \cite{Calandra2016}, leverage uncertainty information to direct the robot to optimal performance. In \cite{Silverio2018a} we took advantage of uncertainty to regulate robot stiffness, in order to make it compliant (and safer) when uncertain about its actions. However, to the best of our knowledge, variability and uncertainty have never been exploited simultaneously in imitation learning.

\begin{figure}
	\includegraphics[width=\columnwidth]{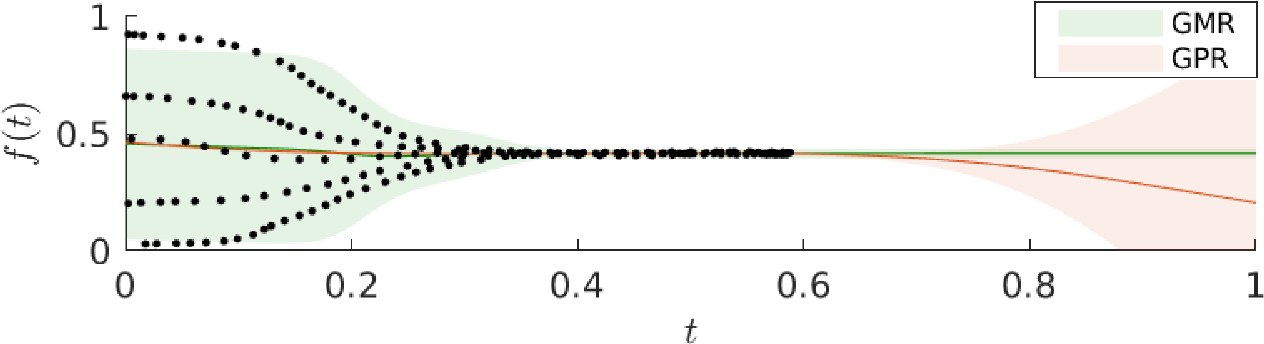}
	\caption{Gaussian mixture regression (GMR) and Gaussian process regression (GPR) provide complementary notions of variance (represented as green and red shaded areas) as variability and absence of training datapoints (depicted as black dots). With a unified technique, robots can learn controllers that are modulated by both types of information.}
	\label{fig:GMRvsGPR}
	%	\vspace{-0.5cm}
\end{figure}

In this paper we introduce an approach that predicts \textit{variability}, \textit{correlations} and \textit{uncertainty} from KMP and uses this information to design optimal controllers from demonstrations. These drive the robot with high precision when the variability in the data is low (while respecting the observed correlations across degrees of freedom) and render the robot compliant (and safer to interact with) when the uncertainty is high. The uncertainty is further leveraged by the robot to know when different controllers, responsible for the execution of separate, elementary sub-tasks, should be activated. In particular we:

\begin{enumerate}
	\item demonstrate that KMP predicts full covariance matrices and uncertainty (Sections \ref{sec:KMP} and \ref{sec:uncertKMP})
	\item exploit a linear quadratic regulator (LQR) formulation that yields control gains which are a function of both covariance \textit{and} uncertainty (Section \ref{sec:optimalCtr})
	\item dovetail 1), 2) with the concept of \textit{fusion of controllers} \cite{Silverio2018b} which allows for demonstrating one complex task as separate sub-tasks, whose activation depends on individual uncertainty levels (Section \ref{sec:fusion})
\end{enumerate}

Experimentally, we expand on a previously published robot-assisted painting scenario and validate the approach using a KUKA LWR where different types of controllers are used for individual sub-tasks (Section \ref{sec:experiments}). We provide a discussion of the approach and the obtained results in Section \ref{sec:discussion} and concluding remarks and possible extensions in Section \ref{sec:conclusions}.

\section{Related Work}
\label{sec:relatedWork}

Most probabilistic regression techniques provide variance predictions in some form. GMR, relying on a previously trained GMM, computes full covariance matrices encoding the correlation between output variables. However, it does not measure uncertainty, defaulting to the covariance of the closest Gaussian component when a query point is far from the model. GPR, despite estimating uncertainty, assumes constant noise  therefore not taking the variability of the outputs  into account. Heteroscedastic Gaussian Processes (HGP) \cite{Goldberg1998,Kersting2007} introduce an input-dependent noise model into the regression problem. Nonetheless, tasks with multiple outputs require the training of separate HGP models, thus output correlations are not straightforward to learn in the standard formulation. In addition, the noise is treated as a latent function, hence each HGP depends upon the definition of two Gaussian processes (GP) per output, scaling poorly with the number of outputs. 
%Another potentially valid approach is Generalized Wishart Processes (GWP) \cite{Wilson10}, which are aimed at data-driven regression of covariance matrices. Their structure relies on the external product of GP, requiring the definition and training of one individual GP per output. Similarly to HGP, GWP treat variability and correlations as unobserved, latent variables, leading to costly parameter estimation for adequate regressions. Moreover, despite having a GP-based structure, it is not clear how the information contained in each GP can be translated into uncertainty.
In \cite{Choi2018}, Choi \etal propose to use mixture density networks (MDN) in an imitation learning context to predict both variability and uncertainty. The main drawback of the approach, similarly to HGP, is that outputs are assumed to be uncorrelated. Moreover, in \cite{Choi2018} only the uncertainty is used in the proposed imitation learning framework, without considering variability. As opposed to the aforementioned works, we here show that KMP predicts both full covariance matrices and a diagonal uncertainty matrix, parameterized by its hyperparameters, allowing the access to all the desired information. Table \ref{tab:RelatedWork} details the differences between variance predictions of different algorithms, highlighting that KMP estimates all desired features in our approach. 

In terms of estimating optimal controllers from demonstrations, previous works have either exploited full covariance matrices encoding variability and correlations \cite{Calinon14ICRA,Medina2012,Lee2015} or diagonal uncertainty matrices \cite{Silverio2018a}. While the former are aimed at control efficiency, by having the robot apply higher control efforts where required (depending on variability), the latter target safety, with the robot becoming more compliant when uncertain about its actions. The LQR we propose in Section \ref{sec:optimalCtr} is identical to the one in \cite{Calinon14ICRA,Silverio2018a,Huang2018}. However, by benefiting from the KMP predictions, it unifies the best of the two approaches. Umlauft \etal \cite{Umlauft17} propose a related formulation where, using Wishart processes, they build full covariance matrices with uncertainty. However their solution requires a very high number of parameters, whose estimation relies heavily on optimization, and their control gains are set heuristically.

Finally, inspired by \cite{Calinon09AR}, in \cite{Silverio2018b} we proposed a \textit{fusion of controllers} to allow robots  to smoothly switch between sub-tasks based on the uncertainty of each sub-task's controller. Here we go one step further and consider \textit{optimal} controllers learned from demonstrations into the fusion, instead of manually defining the control gains. In previous work \cite{Huang2018}, we have studied the fusion of optimal controllers. However, in that case we focused on time-driven trajectories whereas here we consider multi-dimensional inputs and uncertainty.

The approach described in the next sections therefore aims at a seamless unification of concepts exploited in previous work, taking imitation learning one step ahead into the learning of optimal controllers for potentially complex tasks.

\begin{table}
	% Note to self: \rule command to increase vertical space and the customized \Cline
	\centering
	\begin{tabular}{cccc}
		& \multicolumn{3}{c}{Types of prediction}\\
		\Cline{1pt}{2-4}
		\rule{0pt}{2.5ex} & Variability &  Uncertainty & Correlations \\
		\toprule[1.0pt]
		GMM/GMR \cite{Calinon14ICRA} & \checkmark & -- & \checkmark \\
		GPR \cite{Rasmussen06} & -- & \checkmark & --\\
		HGP \cite{Goldberg1998,Kersting2007} & \checkmark  & \checkmark & -- \\
		%GWP \cite{Wilson10} & \checkmark  & ? & \checkmark \\
		MDN \cite{Choi2018} & \checkmark & \checkmark & -- \\		
		Our approach & \checkmark & \checkmark & \checkmark\\		
		\bottomrule[1.0pt]
	\end{tabular}
	\caption{(Co)variance predictions of different techniques.}
	\label{tab:RelatedWork}
%	\vspace{-0.5cm}
\end{table}

\section{Kernelized Movement Primitives}
\label{sec:KMP}

%In this section we review the Kernelized Movement Primitive (KMP) formulation \cite{Huang2017}.
We consider datasets comprised of $ H $ demonstrations with length $ T $, $ \{\{\mb{\xi}^{t,h}_\mathcal{I},\mb{\xi}^{t,h}_\mathcal{O}\}^T_{t=1}\}^H_{h=1}$ where ${\mb{\xi}_{\mathcal{I}} \in \mathbb{R}^{D_\mathcal{I}}}$  and  $\mb{\xi}_{\mathcal{O}} \in \mathbb{R}^{D_\mathcal{O}}$ denote inputs and outputs ($ \mathcal{I}, \mathcal{O} $ are initials for `input' and `output'), respectively, and $ D_\mathcal{I}, D_\mathcal{O} $ are their dimensions. $ \mb{\xi}_\mathcal{I} $ can represent any variable of interest to drive the movement synthesis (e.g., time, object/human poses) and
$ \mb{\xi}_\mathcal{O} $ encodes the desired state of the robot (e.g., an end-effector position, a joint space configuration). 
KMP assumes access to an $ N$-dimensional probabilistic trajectory distribution $ \{\mb{\xi}_\mathcal{I}^n, \hat{\mb{\mu}}_n,\hat{\mb{\Sigma}}_n\}^N_{n=1} $ mapping a sequence of inputs to their corresponding means and covariances, which encompass the important features in the demonstration data. This probabilistic reference trajectory can be obtained in various ways, for example by computing means and covariances empirically at different points in a dataset or by using unsupervised clustering techniques. Here we follow the latter direction, in particular by using a GMM to cluster the data and GMR to obtain the trajectory distribution that initializes KMP (done once after data collection).

By concatenating the trajectory distribution into $ {\mb{\mu} = [\hat{\mb{\mu}}^\trsp_1 \> \ldots \> \hat{\mb{\mu}}^\trsp_N]^\trsp}$ and	${\mb{\Sigma} = \mathrm{blockdiag}(\hat{\mb{\Sigma}}_1,\ldots,\hat{\mb{\Sigma}}_N)}$,
%
%\begin{align}
%	\mb{\mu} & = [\hat{\mb{\mu}}^\trsp_1 \quad \ldots \quad \hat{\mb{\mu}}^\trsp_N]^\trsp,\\
%	\mb{\Sigma} & = \mathrm{blockdiag}(\hat{\mb{\Sigma}}_1,\ldots,\hat{\mb{\Sigma}}_N),
%\end{align} 
%
KMP predicts a new Gaussian distribution at new test points $\mb{\xi}_\mathcal{I}^*$ according to \cite{Huang2019d}
\begin{align}
%	\mathbb{E}(\mb{\xi}^*_\mathcal{O}) & = \mb{k}^*(\mb{K} + \lambda_1\mb{\Sigma})^{-1}\mb{\mu}, \label{eq:KMPmean}\\
%	\mathbb{V}(\mb{\xi}^*_\mathcal{O}) & = \frac{N}{\lambda_2}\left(\mb{k}^{**}-\mb{k}^*(\mb{K} + \lambda_2\mb{\Sigma})^{-1})\mb{k}^{*\trsp}\right). \label{eq:KMPcov}
	\mb{\mu}^*_\mathcal{O} & = \mb{k}^*(\mb{K} + \lambda_1\mb{\Sigma})^{-1}\mb{\mu}, \label{eq:KMPmean}\\
	\mb{\Sigma}^*_\mathcal{O} & = \frac{N}{\lambda_2}\left(\mb{k}^{**}-\mb{k}^*(\mb{K} + \lambda_2\mb{\Sigma})^{-1})\mb{k}^{*\trsp}\right). \label{eq:KMPcov}
\end{align}
where 
\begin{equation}
	\mb{K} = 
	\left[\begin{matrix}
		\mb{k}(\mb{\xi}^1_\mathcal{I},\mb{\xi}^1_\mathcal{I}) & \cdots & \mb{k}(\mb{\xi}^1_\mathcal{I},\mb{\xi}^N_\mathcal{I})\\
		\vdots & \ddots & \vdots\\
		\mb{k}(\mb{\xi}^N_\mathcal{I},\mb{\xi}^1_\mathcal{I}) & \cdots & \mb{k}(\mb{\xi}^N_\mathcal{I},\mb{\xi}^N_\mathcal{I})\\
	\end{matrix}\right]
\end{equation}
is a matrix evaluating a chosen kernel function $ k(.,.) $ at the training inputs, $ \mb{k}^* = \left[ \mb{k}(\mb{\xi}^*_\mathcal{I},\mb{\xi}_\mathcal{I}^1) \> \ldots \> \mb{k}(\mb{\xi}^*_\mathcal{I},\mb{\xi}_\mathcal{I}^N) \right]$ and $ \mb{k}^{**} = \mb{k}(\mb{\xi}^*_\mathcal{I},\mb{\xi}_\mathcal{I}^*)$. Moreover, $ \mb{k}(\mb{\xi}^i_\mathcal{I},\mb{\xi}^j_\mathcal{I}) = k(\mb{\xi}^i_\mathcal{I},\mb{\xi}^j_\mathcal{I})\mb{I}_{D_\mathcal{O}} $. Hyperparameters $ \lambda_1,\lambda_2 $ are regularization terms chosen as to constrain the magnitude of the predicted mean and covariance, respectively. The kernel treatment implicit in \eqref{eq:KMPmean}-\eqref{eq:KMPcov} assumes the previous choice of a kernel function that depends on the characteristics of the training data. We here consider the squared-exponential kernel
\begin{equation}
	k(\mb{\xi}^i_\mathcal{I},\mb{\xi}^j_\mathcal{I}) = \sigma^2_f\>\mathrm{exp}\left(-\frac{1}{l}||\mb{\xi}^i_\mathcal{I}-\mb{\xi}^j_\mathcal{I}||^2\right),
	\label{eq:kernel}
\end{equation}
a common choice in the literature. We hence have that KMP with kernel \eqref{eq:kernel} requires the definition of four hyperparameters $ \{\lambda_1, \lambda_2, l, \sigma_f^2\} $. Note the similarity between predictions \eqref{eq:KMPmean}-\eqref{eq:KMPcov} and other  kernel-based techniques (e.g. GPR, HGP). The main difference is that in KMP the \textit{noise model} is learned through $ \mb{\Sigma} $ which describes both the variability and correlations present in the data throughout the trajectory. This makes KMP a richer representation when compared to GPR or HGP, which assume either constant noise ${ \hat{\mb{\Sigma}}_i=\sigma^2_\epsilon\mb{I}_{D_\mathcal{O}},\forall i=1,\ldots,N} $ (GPR) or input-dependent uncorrelated noise ${\hat{\mb{\Sigma}}_i = \sigma^2_\epsilon(\mb{\xi}^i_\mathcal{I})\mb{I}_{D_\mathcal{O}}}$ (HGP).

\section{Uncertainty-aware imitation learning with KMP}
\label{sec:uncertainApproach}

We now demonstrate that KMP provides an estimation of uncertainty through \eqref{eq:KMPcov}, by defaulting to a diagonal matrix completely specified by its hyperparameters in the absence of datapoints (Section \ref{sec:uncertKMP}). In addition we propose a control framework to convert the predictions into optimal robot actions (Section \ref{sec:optimalCtr}) and the fusion of optimal controllers (Section \ref{sec:fusion}). 

\subsection{Uncertainty predictions with KMP}
\label{sec:uncertKMP}

In the light of the kernel treatment \eqref{eq:KMPcov} and the exponential kernel \eqref{eq:kernel}, both covariance and uncertainty predictions emerge naturally in the KMP formulation. While the former occur within the training region, the latter arise when querying the model away from the original data.

\begin{lemma}
	The squared exponential kernel \eqref{eq:kernel} goes to zero as $ || \mb{\xi}^n_\mathcal{I} - \mb{\xi}^*_\mathcal{I}|| \rightarrow +\infty, \forall n=1,\ldots,N$.
	\label{lemma1}
\end{lemma}
\begin{proof}
	Let us consider $ d = ||\mb{\xi}^{\hat{n}}_\mathcal{I} - \mb{\xi}^*_\mathcal{I} || $, where $ {\hat{n} = \argmin_n ||\mb{\xi}^n_\mathcal{I} - \mb{\xi}^*_\mathcal{I} ||} $ is the index of the training point with the minimum distance to the test point $\mb{\xi}^*_\mathcal{I}$.
	\begin{equation}
		\lim\limits_{d\rightarrow +\infty}k(\mb{\xi}^{\hat{n}}_\mathcal{I},\mb{\xi}^*_\mathcal{I}) = \lim\limits_{d\rightarrow +\infty}\sigma^2_f\mathrm{exp}(-\frac{1}{l} d^2) = 0.
	\end{equation}
\end{proof}

Lemma \ref{lemma1} extends to other popular exponential kernels, including the Mat\'ern kernel \cite{Rasmussen06}.

\begin{theorem}
	Covariance predictions \eqref{eq:KMPcov} converge to a diagonal matrix completely specified by the KMP hyperparameters as test inputs $ \mb{\xi}^*_\mathcal{I} $ move away from the training dataset, i.e. $ d \rightarrow +\infty$. Particularly,
	\begin{equation}
		\lim\limits_{d\rightarrow +\infty}\mb{\Sigma}^*_\mathcal{O} = \sigma_f^2\>\frac{N}{\lambda_2} \mb{I}_{D_\mathcal{O}}.
		\label{eq:uncertainty}
	\end{equation}	
	\vspace{0.05cm}
\end{theorem}

\begin{proof}
	Following from Lemma \ref{lemma1} and knowing that $ \mb{k}^* = \left[ \mb{k}(\mb{\xi}^*_\mathcal{I},\mb{\xi}_\mathcal{I}^1) \> \ldots \> \mb{k}(\mb{\xi}^*_\mathcal{I},\mb{\xi}_\mathcal{I}^N) \right]$ we have  ${ \lim\limits_{d\rightarrow +\infty}\mb{k}^*=\mb{0}_{D_\mathcal{O}\times ND_\mathcal{O}} }$. Hence
	\begin{equation}
		\lim\limits_{d\rightarrow +\infty}\mb{\Sigma}^*_\mathcal{O} = \lim\limits_{d\rightarrow+\infty}\frac{N}{\lambda_2} \mb{k}^{**}.
		\label{eq:KMPlimk}
	\end{equation}
Moreover we have
\begin{equation}
	\mb{k}^{**} = \mb{k}\left(\mb{\xi}^*_\mathcal{I},\mb{\xi}^*_\mathcal{I}\right) = \sigma_f^2 \mathrm{exp}\left(-\frac{1}{l}0 \right) = \sigma_f^2\mb{I}_{{D}_\mathcal{O}},
	\nonumber
\end{equation}
which replaced in \eqref{eq:KMPlimk} yields \eqref{eq:uncertainty}.
%	\begin{equation}
%	\lim\limits_{d\rightarrow +\infty}\mathbb{V}(\mb{\xi}^*_\mathcal{O}) = \sigma_f^2\>\mathrm{diag}\!\left(\frac{N}{\lambda_2}\right). 
%	\end{equation}
\end{proof}

Equation \eqref{eq:uncertainty} plays a crucial role in our approach. It provides a principled way to know when the model is being queried in regions where data was not present during training. We leverage this information to 1) make the robot compliant when unsure about its actions and 2) let the robot know when to execute control actions pertaining to different KMPs. Moreover, through the dependence on $ \sigma^2_f$, $ N  $ and $\lambda_2$, one can adjust the expression of uncertainty provided by the model, through the tuning of any of those hyperparameters. For instance, increasing the length of the initialized trajectory distribution $ N $ has the effect of scaling the uncertainty. GPR offers a similar property, where the variance prediction converges to the scalar $ \sigma^2_f $. However this is rather limiting as tuning this hyperparameter can have undesired effects on the mean prediction. In KMP, $ N $ and $ \lambda_2 $ do not affect the mean prediction as they do not parameterize the kernel function. Moreover, \eqref{eq:KMPcov} is typically robust to their choice, providing freedom for tuning while yielding proper predictions  (see \cite{Huang2019d} for details).

%\subsection{Incorporating uncertainty into a probabilistic LfD model}
%
%We here introduce a measure of confidence based on the estimated uncertainty:
%
%\begin{equation}
%	c_i(t) = 1 - \frac{\epsilon^2_i(t)}{\epsilon^2_n}
%	\label{eq:confidence}
%\end{equation}
%
%We exploit \eqref{eq:confidence} to construct an \textit{importance matrix} which scales the covariance according to the uncertainty in the prediction:
%
%\begin{equation}
%	\mb{\Lambda}(t) = \left[\begin{matrix}
%		c_1(t) & \mb{0} & \mb{0} \\
%		\mb{0} & \ddots & \mb{0} \\
%		\mb{0} & \mb{0} & c_N(t) \\
%	\end{matrix}\right]
%\end{equation}
%
%where the re-scaling is achieved via 
%
%\begin{equation}
%	\mb{\Sigma}(t) \leftarrow \mb{\Lambda}^{-1^\trsp}\mb{\Sigma}(t)\mb{\Lambda}^{-1}
%	\label{eq:covariance_rescale}
%\end{equation}
%
%Moreover a regularization can be added to the importance matrix to avoid numerical problems in the inverses via
%
%\begin{equation}
%	\mb{\Lambda}(t) \leftarrow \mb{\Lambda} + \lambda\mb{I}
%	\label{eq:imp_reg}
%\end{equation}
%
%where the importance is 
%

\subsection{Computing optimal controllers from KMP }
\label{sec:optimalCtr}

We now propose to use $ \mb{\Sigma}^*_\mathcal{O} $ to obtain variable control gains that result in a compliant robot both when the variability and uncertainty are high\footnote{In the context of movement synthesis, new inputs occur at every new time step thus we will replace $ * $ by $ t $  from now on in the notation.}. We follow the concept introduced in \cite{Medina2012} and formulate the problem as a LQR. Let us consider linear systems $\mb{\dot{\zeta}}_t = \mb{A}\mb{\zeta}_t + \mb{B}\mb{u}_t$, where ${ \mb{\zeta}_t, \mb{\dot{\zeta}}_t \in \mathbb{R}^{N_S} }$ denote the system state at time $ t $ and its first-order derivative ($ N_S $ is the dimension of the state) and ${\mb{u}_t\in\mathbb{R}^{N_C}}$ is a control command, where $ N_C $ denotes the number of controlled degrees of freedom. {\color{black}Moreover, $ \mb{A}\in\mathbb{R}^{N_S\times N_S} $ and $\mb{B}\in\mathbb{R}^{N_S\times N_C} $ represent the state and input matrices}. We will stick to task space control and hence make a simplifying assumption, in line with \cite{Calinon14ICRA}, that the end-effector can be modeled as a unitary mass, yielding a double integrator system
\begin{equation}
\mb{A}=\left[\begin{matrix}  \mb{0} & \mb{I} \\ \mb{0} & \mb{0}\end{matrix}\right], \>\>\> \mb{B} = \left[\begin{matrix} \mb{0} \\ \mb{I}\end{matrix}\right],
\end{equation}
where $ \mb{0} $ and $ \mb{I} $ are zero and identity matrices of appropriate dimension. We define the end-effector state at $ t $ as its Cartesian position and velocity $ \mb{x}_t, \dot{\mb{x}}_t $, i.e. $ {\mb{\zeta}_t = [\mb{x}^\trsp_t\>\> \dot{\mb{x}}^\trsp_t]^\trsp}$, and therefore $ \mb{u}_t $ corresponds to acceleration commands.

At every time step $ t $ of a task, a KMP is queried with an input test point $ \mb{\xi}^t_\mathcal{I} $, predicting a mean $ \mb{\mu}^{t}_\mathcal{O} $ and a covariance matrix $ \mb{\Sigma}^{t}_\mathcal{O} $. We define $ \hat{\mb{\zeta}}_t=\mb{\mu}^{t}_\mathcal{O} $, i.e. the desired state for the end-effector is given by the mean prediction of KMP. For time-driven tasks, where $ \mb{\xi}^t_\mathcal{I} = t $, a sequence of reference states $ \mb{\hat{\zeta}}_{t=1,\ldots,T} $ can be easily computed and an optimal control command $ \mb{u}_t $ can be found, minimizing
\begin{equation}
%	\mb{J} = \sum^{\infty}_{t=1} (\mb{\hat{\xi}}_t - \mb{\xi}_t)^\trsp \mb{Q}_t (\mb{\hat{\xi}}_t - \mb{\xi}_t) + \mb{u}^\trsp_t \mb{R}_t \mb{u}_t
\mb{c}(t) = \sum^{T}_{t=1} (\mb{\hat{\zeta}}_t - \mb{\zeta}_t)^\trsp \mb{Q}_t (\mb{\hat{\zeta}}_t - \mb{\zeta}_t) + \mb{u}^\trsp_t \mb{R}_t \mb{u}_t,
\label{eq:LQRobj}
\end{equation}
where $ \mb{Q}_t $ is a $ N_S\times N_S $ positive semi-definite matrix that determines how much the optimization penalizes deviations from $ \mb{\hat{\zeta}}_t $ and $ \mb{R}_t $ is an $ N_C \times N_C  $ positive-definite matrix that penalizes the magnitude of the control commands. Equation \eqref{eq:LQRobj} is the cost function of the finite horizon LQR and its solution is obtained through backward integration of the Riccati equations (see \cite{Calinon14ICRA}). In non-time-driven tasks, e.g. when $ \mb{\xi}^t_\mathcal{I} $ is the state of a human that collaborates with the robot, it is not straightforward to predict a sequence of desired states. In these cases, we resort to the infinite horizon formulation
\begin{equation}
%	\mb{J} = \sum^{\infty}_{t=1} (\mb{\hat{\xi}}_t - \mb{\xi}_t)^\trsp \mb{Q}_t (\mb{\hat{\xi}}_t - \mb{\xi}_t) + \mb{u}^\trsp_t \mb{R}_t \mb{u}_t
\mb{c}(t) = \sum^{\infty}_{n=t} (\mb{\hat{\zeta}}_t - \mb{\zeta}_n)^\trsp \mb{Q}_t (\mb{\hat{\zeta}}_t - \mb{\zeta}_n) + \mb{u}^\trsp_n \mb{R}_t \mb{u}_n.
\label{eq:LQRobjInf}
\end{equation}
which is solved iteratively using the algebraic Riccati equation. In both cases \eqref{eq:LQRobj}, \eqref{eq:LQRobjInf}, the solution is given by a linear state feedback control law
\begin{equation}
\mb{u}_t = [\mb{K}_t^\mathcal{P}\>\>\>\mb{K}_t^\mathcal{V}](\mb{\hat{\zeta}}_t - \mb{\zeta}_t)
\label{eq:linsys}
\end{equation}
where $ \mb{K}_t^\mathcal{P}, \mb{K}_t^\mathcal{V} $ are stiffness and damping gain matrices that drive the system to the desired state. 

\begin{algorithm}[bt]
	\small
	\captionsetup{font=small}
	\caption{Uncertainty-aware imitation learning}
	\begin{algorithmic}[1]
		\Statex \textit{\textbf{Initialization}}
		\State Identify number of sub-tasks $ P $
		\State Collect demonstrations $ \{\{\{\mb{\xi}^{t,h,p}_\mathcal{I},\mb{\xi}^{t,h,p}_\mathcal{O}\}^T_{t=1}\}^H_{h=1}\}^P_{p=1}$
		\State Generate trajectory distributions $ \{\{\mb{\xi}_\mathcal{I}^{n,p}, \hat{\mb{\mu}}_{n,p},\hat{\mb{\Sigma}}_{n,p}\}^N_{n=1}\}^P_{p=1} $
		\State Select hyperparameters $ \{\sigma^2_{f,p}, l_p, \lambda_{1,p}, \lambda_{2,p} \}^P_{p=1}$ and $ \mb{R}_p $ 
	\end{algorithmic}
	\begin{algorithmic}[1]
		\Statex \textit{\textbf{Movement synthesis}}
		\State \textit{Input:} Test point $ \mb{\xi}^t_\mathcal{I} $
		\For{$p=1,\dots,P$}
		\State Compute $ \mb{\mu}^{t,p}_\mathcal{O}, \mb{\Sigma}^{t,p}_\mathcal{O},$ per \eqref{eq:KMPmean}, \eqref{eq:KMPcov}
		\State Set $ \hat{\mb{\zeta}}^p_t=\mb{\mu}^{t,p}_\mathcal{O} $ and $ \mb{Q}^{p}_t = (\mb{\Sigma}^{t,p}_\mathcal{O})^{-1} $ 
		\State Find optimal gains $ \mb{K}_{t,p}^\mathcal{P}, \mb{K}_{t,p}^\mathcal{V} $  and compute $ \mb{u}^{p}_t $ per \eqref{eq:u}
		\State Set $ \mb{\Gamma}^{p}_t = (\mb{\Sigma}^{t,p}_\mathcal{O})^{-1} $
		\EndFor
		\State Compute $ \hat{\mb{u}}_t $ from \eqref{eq:optProb}
		\State \textit{Output:} Control command $ \hat{\mb{u}}_t $
	\end{algorithmic}
	\label{alg:uncertainIL}
\end{algorithm}

Finally, we set ${ \mb{Q}_t = (\mb{\Sigma}^{t}_\mathcal{O})^{-1} }$. Unlike in previous  works where a similar choice is made \cite{Calinon14ICRA, Silverio2018a, Medina2012}, in our approach the unique properties of KMP endow the robot with the ability to modulate its behavior in the face of two different conditions as a consequence of this setting. First, when the KMP is queried within the training region, full covariance matrices encoding variability and correlations in the demonstrations are estimated, resulting in control gains that reflect the structure in the data. The robot is hence more precise where variability is low (higher gains) and responds to perturbations according to the observed correlations. Second, as the test input deviates from the training data, the robot becomes increasingly more compliant, as a consequence of \eqref{eq:uncertainty}. Intuitively, this makes sense as the robot should be safe when the uncertainty about is actions increases, which is achieved in our formulation by an automatic decrease of the control gains. Our approach is hence the first to permit the robot to be optimal in the region where demonstrations were provided, and safe where data is absent.
\begin{figure}
	\includegraphics[width=\columnwidth]{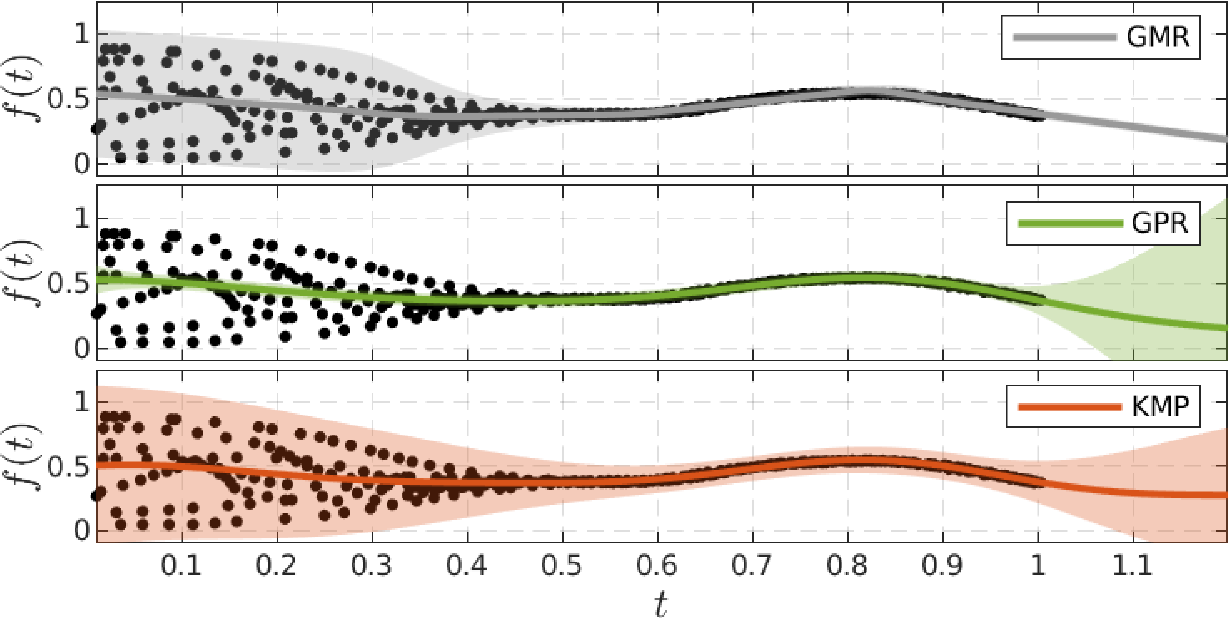}
	%	\vspace{-0.5cm}
	\caption{Comparison between KMP, GMR and GPR. Datapoints are plotted in black, solid lines represent the mean and shaded areas correspond to two standard deviations (computed from respective variance estimations).}
	%	\vspace{-0.5cm}
	\label{fig:KMPvsAll}	
\end{figure}

\subsection{Fusing optimal controllers}
\label{sec:fusion}
%Often a task requires that different controllers act on a robot to perform different sub-tasks. Consider a robot manipulating an object while keep a natural joint posture -- the former is ensured by a task space controller while the latter by a joint space one. Moreover,
It is often convenient to split the demonstration of complex tasks into smaller, less complex sub-tasks (e.g. grasping a tool and avoiding obstacles). Here we adapt the previously introduced notion of \textit{fusion of controllers} \cite{Silverio2018b} to account for optimal controllers, such as those described in Section \ref{sec:optimalCtr}. Let us consider $ p=1,\ldots,P $ candidate controllers generating commands $ \mb{u}^{p}_t $ that may act on the robot at every time step $ t $ (we omit the subscript $ t $ in the remainder of this section). In a fusion of controllers, an optimal command is computed as
\begin{equation}
\mb{\hat{u}} = \arg\underset{\mb{u}}{\min} \sum_{p=1}^{P}\left(\mb{u} - \mb{u}^{p}\right)^{\trsp} \!\mb{\Gamma}^{p}\! \left(\mb{u} - \mb{u}^{p}\right),
\label{eq:optProb}
\end{equation}
where $ \mb{\Gamma}^{p} $ are weight matrices that regulate the contribution of each individual controller. Examples of $ \mb{\Gamma}^{p} $ found in the literature include scalar terms that maximize external rewards \cite{Dehio2015} and precision matrices, either computed from covariance \cite{Silverio2018b, Calinon09AR} or uncertainty \cite{Silverio2018a}. Equation \eqref{eq:optProb} has an analytical solution given by $ \mb{\hat{u}} = \mb{\hat{\Sigma}}_{u}\sum_{p=1}^{P} \bluetext{\mb{\Gamma}^{p}} \mb{u}^{p}, $
%
%\begin{equation}
%\mb{\hat{u}} = \mb{\hat{\Sigma}}_{u}\sum_{p=1}^{P} \bluetext{\mb{\Gamma}^{p}} \mb{u}^{p},
%\label{eq:GaussProdTorques}
%\end{equation}
%
where $ \mb{\hat{\Sigma}}_{u} = \Big(\sum_{p=1}^{P} \bluetext{\mb{\Gamma}^{p}} \Big)^{-1} $. When $ \mb{u}^p $ and  $ \mb{\Gamma}^p $ are viewed as the mean and precision matrix of a Gaussian distribution, this solution corresponds to a Gaussian product \cite{Silverio2018b}.

\begin{figure*}
	\centering
	\begin{subfigure}[b]{0.7435\textwidth}
		\includegraphics[width=\textwidth]{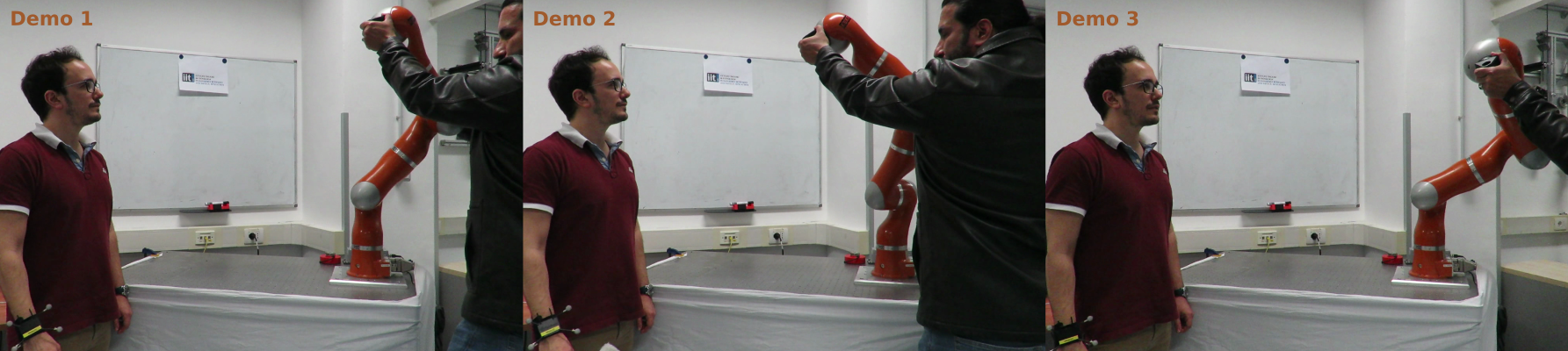}
		\caption{Examples of handover starting positions.}
		\label{fig:demostask1-1}		
	\end{subfigure}
	\begin{subfigure}[b]{0.248\textwidth}
		\includegraphics[width=\textwidth]{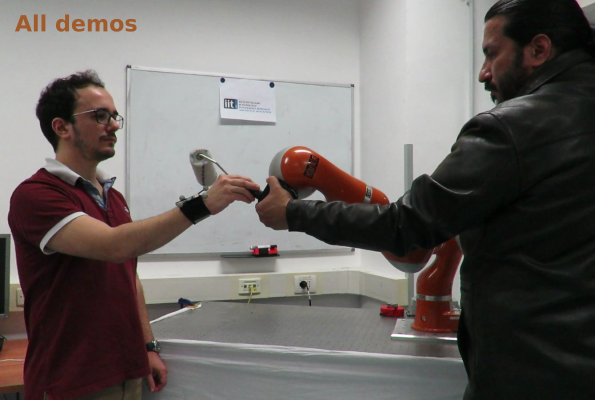}
		\caption{Handover end position.}
		\label{fig:demostask1-2}
	\end{subfigure}
	%	\vspace{-0.5cm}
	\caption{Handover demonstrations. The robot, starting from different initial positions, is moved kinesthetically towards a handover location, where the human hands it a paint roller.}
	\label{fig:demostask1-(12)}
	%	\vspace{-0.5cm}
\end{figure*}

%In \cite{Silverio2018b} we have treated $ \mb{u}^p,\mb{\Gamma}^p $ as distributions of torque control commands. These were computed from probabilistic references (positions and velocities) with mean $ \mb{\mu}^p_\mathcal{O} $ and covariance $ \mb{\Sigma}^p_\mathcal{O} $ (obtained from GMR), modulated by the control gains. The precision matrix was in this case given by
%
%\begin{equation}
%	\mb{\Gamma}^{p} = \left(\mb{L}^p\mb{\Sigma}^{p}_\mathcal{O}\mb{L}^{p\>\trsp}\right)^{-1},
%	\label{eq:oldgains}
%\end{equation}
%
%where $ \mb{L}^p $ was a matrix of \emph{fixed}, manually set gains for controller $ p $. 

%In order to fuse optimal controllers, we define an LQR for each controller $ p $, as in Section \ref{sec:optimalCtr}, where 
%
%\begin{equation}
%	\mb{Q}^p = (\mb{\Sigma}^{p}_\mathcal{O})^{-1}.
%	\label{eq:newQ}
%\end{equation}
%
%However, when it comes to parameterizing \eqref{eq:optProb}, Eq. \eqref{eq:oldgains} is an improper choice, when considering such kind of optimal controllers.  
%Because of \eqref{eq:newQ}, optimal gains scale inversely with covariance matrices, i.e. $ \mb{L}^p\propto (\mb{\Sigma}^p_\mathcal{O})^{-1} $. This would have the effect of canceling out $ \mb{\Sigma}^{p}_\mathcal{O} $ in \eqref{eq:oldgains} (e.g. a small $ \mb{\Sigma}^{p}_\mathcal{O} $ would become greatly enlarged by the product $ \mb{L}^p\mb{\Sigma}^{p}_\mathcal{O}\mb{L}^{p\>\trsp} $).
We here propose to use
\begin{equation}
\mb{u}^p = [\mb{K}_{t,p}^\mathcal{P}\>\>\>\mb{K}_{t,p}^\mathcal{V}](\mb{\hat{\zeta}}^p_t - \mb{\zeta}^p_t),
\label{eq:u}
\end{equation}
where $ [\mb{K}_{t,p}^\mathcal{P}\>\>\>\mb{K}_{t,p}^\mathcal{V}] $ are optimal gains estimated using LQR, given the KMP of controller $ p $, and
\begin{equation}
\mb{\Gamma}^p = \left(\mb{\Sigma}^p_\mathcal{O}\right)^{-1}.
\label{eq:Gamma}
\end{equation}
As a consequence of \eqref{eq:Gamma}, controllers with high uncertainty will have negligible influence in the resulting command computed from \eqref{eq:optProb}. This permits the demonstration of a task into sub-tasks, whose activation during reproduction depends on their individual uncertainties.
%
%To summarize, Eqs. \eqref{eq:newQ}, \eqref{eq:Gamma} improve \cite{Silverio2018b} to account for optimal controllers, particularly:
%
%\begin{enumerate}
%	\item the control commands $ \mb{u}^p $ are now optimal and are computed from \eqref{eq:linsys} by setting \eqref{eq:newQ}, 
%	\item the weight matrix is now given by the precision matrix of the KMP prediction \eqref{eq:Gamma}.
%\end{enumerate}
%
%It is noteworthy that, in addition to providing a principled way to manage different sub-tasks, the smooth nature of the KMP predictions permits transitions between the sub-tasks without discontinuities in the control command that is send to the robot.
%
Algorithm \ref{alg:uncertainIL} summarizes the complete approach.

\section{Experimental results}
\label{sec:experiments}

In this section we validate our approach from Section \ref{sec:uncertainApproach} using a toy example with synthetic data (Section \ref{sec:toyexample}) and a robot-assisted painting task (Sections \ref{sec:handover} and \ref{sec:fusiontaskspace}). While we have exploited the latter scenario in previous work \cite{Silverio2018b},  here we expand it by considering optimal controllers. The complete task is divided into two sub-tasks: a handover of a paint roller and the application of painting strokes by the robot on a wooden board. In both sub-tasks, the robot motion is driven by the position of the human hand. A supplementary video showing the obtained results is available at \href{http://joaosilverio.weebly.com/uncert.html}{\small\texttt{http://joaosilverio.weebly.com/uncert.html}}

\subsection{1-D regression example with synthetic data}
\label{sec:toyexample}

We first consider the regression of a scalar function. Using an artificially generated dataset we trained a KMP with $ {K = 4} $ (number of Gaussian components used in the initialization GMM), $ \sigma^2_f = 1.0 $ , $ l = 1\times 10^{-2} $, $ \lambda_1 = 5 $ and $ \lambda_2 = 750 $. We sampled a trajectory distribution to initialize the KMP with $ N=750 $ datapoints. Figure \ref{fig:KMPvsAll} shows the original dataset and the approximated function using KMP, GMR\footnote{computed from the GMM that initialized the KMP} and GPR\footnote{with hyperparameters $ l = 1\times 10^{-2} $, $ \sigma^2_\epsilon = 10^{-2} $}. While the three techniques accurately predict the mean trend in the function, the variance prediction given by KMP unifies the predictions from GMR and GPR, approximating the variability of the former and the uncertainty of the latter in the appropriate regions of the input space.

\subsection{Robot Handover}
\label{sec:handover}

\begin{figure}
	\centering
	\includegraphics[width=0.49\columnwidth]{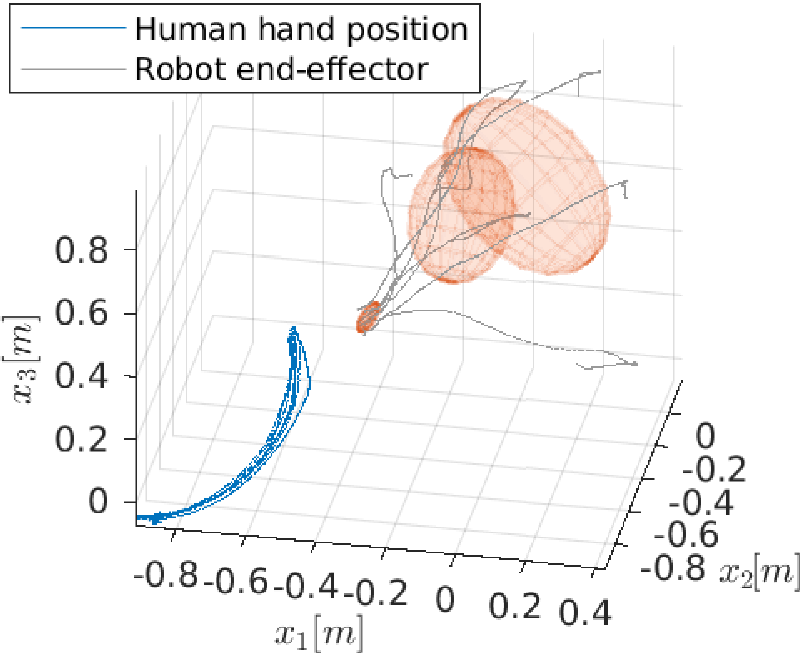}
	\includegraphics[width=0.42\columnwidth]{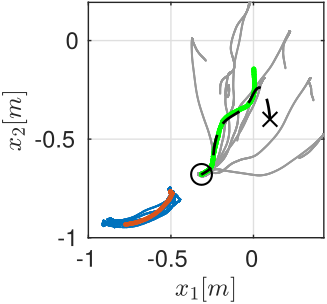}
	\\\vspace{0.3cm}
	\includegraphics[width=0.42\columnwidth]{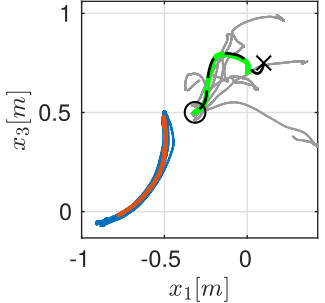}
	\hspace{0.6cm}
	\includegraphics[width=0.42\columnwidth]{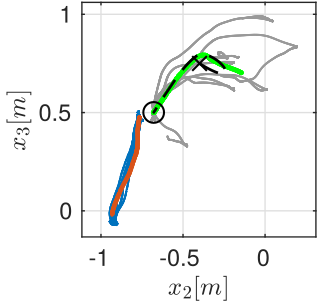}
	\caption{Training data, KMP initialization model and test data. \textbf{Top left:} Demonstrated human hand (blue) and robot end-effector (gray) positions. Red ellipsoids show the 3-component GMM used to initialize KMP. \textbf{Top right and bottom:} Test human hand (orange), KMP generated (green) and robot measured trajectories (black). `$ \times $' and `$ \circ $' mark start and end of trajectory.}
	\label{fig:handover}
%	\vspace{-0.55cm}
\end{figure}

We now show that the proposed approach makes the robot track its reference trajectory using optimal gains near the demonstrations while rendering it compliant when the human is far from the training data. \emph{The handover of the paint roller is achieved by demonstrating to the robot the location of its end-effector as a function of the human hand position.} Note that object handovers are an extensively studied problem in human-robot collaboration and here we simplify the problem to better focus on showcasing our approach. The human hand and robot end-effector positions are here denoted as $ {\mb{x}_H \in \mathbb{R}^3 }$ and $ \mb{x}_R \in \mathbb{R}^3 $ respectively and we wish to learn the mapping $ \mb{x}_H \rightarrow \mb{x}_R$, hence we set $ \mb{\xi}_\mathcal{I} = \mb{x}_H $, $ \mb{\xi}_\mathcal{O} = \mb{x}_R $. We use a KMP with $ K = 3 $, $ \sigma^2_f = 1.0 $, $ l = 0.1 $, $ \lambda_1 = 0.1 $ and $ \lambda_2 = 1 $. Moreover, the KMP is initialized with a trajectory distribution of $ N=500 $ points, obtained using GMR at inputs sampled from the GMM. The cost function of the LQR problem is parameterized with $ \mb{R} = 10^{-2}\mb{I}_{3\times 3} $ and we follow the infinite horizon formulation minimizing \eqref{eq:LQRobjInf}, since the input is the human hand position (i.e. not a time-driven motion).

%\begin{figure}
%	\includegraphics[width=\columnwidth]{figures/handover3x3.eps}
%%	\includegraphics[width=\columnwidth]{figures/handover3D.eps}
%	\caption{Handover.}
%	\label{fig:KMPpredHandover}
%\end{figure}

Figure \ref{fig:handover} shows the training dataset obtained from 7 demonstrations and the resulting GMM used to initialize the KMP (top-left). Moreover it shows the robot end-effector motion computed for a new human hand trajectory used as a test set. As demonstrated, the robot starts at a given position in its task space and moves smoothly towards the handover position, with the learned optimal gains. Figure \ref{fig:Gains} shows the stiffness and damping gains during one execution, plotted as a function of time. The control gains gradually increase as the human hand approaches the robot, ensuring an accurate tracking of the handover position. This goes in opposite direction to the data covariance that starts large and gradually decreases (Fig. \ref{fig:handover} top-left).

\begin{figure}
	\includegraphics[width=\columnwidth]{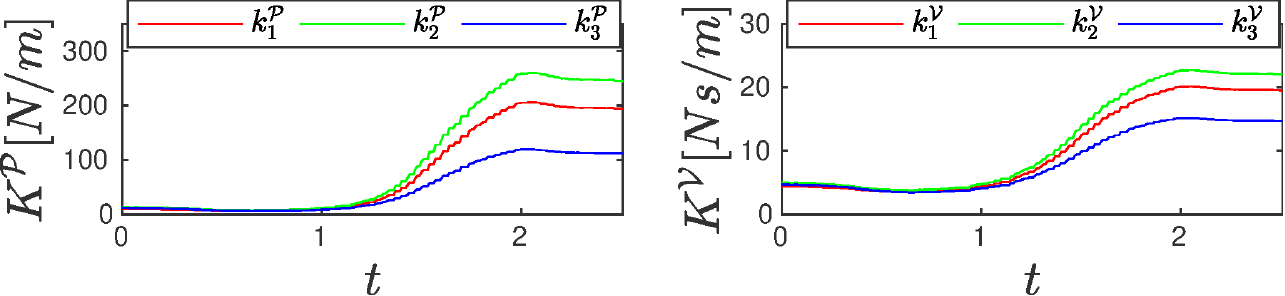}
%	\vspace{-0.5cm}
	\caption{Stiffness and damping gains during one handover. The gains increase towards the end of the task since the end-effector variability decreases as the robot approaches the handover location.}
	\label{fig:Gains}
\end{figure}

\begin{figure}
	\includegraphics[width=\columnwidth]{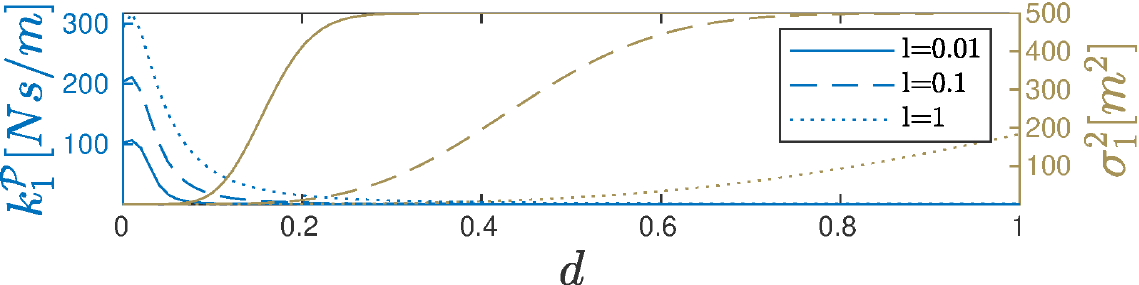}
%	\vspace{-0.65cm}
	\caption{Stiffness gains (blue) and variance (light brown), plotted for the first task space dimension, as a function of the distance to demonstrations $ d $ and different  $ l $. Control gains decrease as the distance increases, making the robot gradually more compliant, hence safer, when it does not know what to do.}
	\label{fig:Lengthscales}
%	\vspace{-0.5cm}
\end{figure}

Figure \ref{fig:Lengthscales} shows the estimated gains (left axis) as one moves away from the region where demonstrations were provided. We manually selected one point in the test set and queried the model at several points up to $ 1m $ away from it along the $ +x_1 $ direction. In order to facilitate the visualization, we plot one single output dimension and omit the damping gains. Notice the increase in the predicted variance (right axis) as one moves away from the demonstrations, which leads to decreasing control gains. This proves experimentally our proposition in Section \ref{sec:uncertKMP}. Moreover, notice the influence of the kernel length scale on how quickly control gains approach $ 0 $. Increasing $ l $ has the effect of decreasing the distance between points, hence higher values result in a slower increase of uncertainty as one moves away from the data. The squared-exponential nature of the kernel therefore permits regulating the rate at which the robot becomes compliant through the tunning of $ l $. The enclosed video further elucidates the compliance aspect of our approach.

\subsection{Fusion of task space controllers}
\label{sec:fusiontaskspace}

In addition to the handover of the paint roller, we also teach the robot how to paint. The goal of this experiment is to show that accessing uncertainty, in addition to covariance, permits the fusion of control commands \emph{in a way that different sub-tasks are activated, depending on the state of the human (here defined by its right hand position)}. In this case, the complete painting task was demonstrated partially into two sub-tasks, whose activation will be inferred from the corresponding models.

\begin{figure}
	\includegraphics[width=\columnwidth]{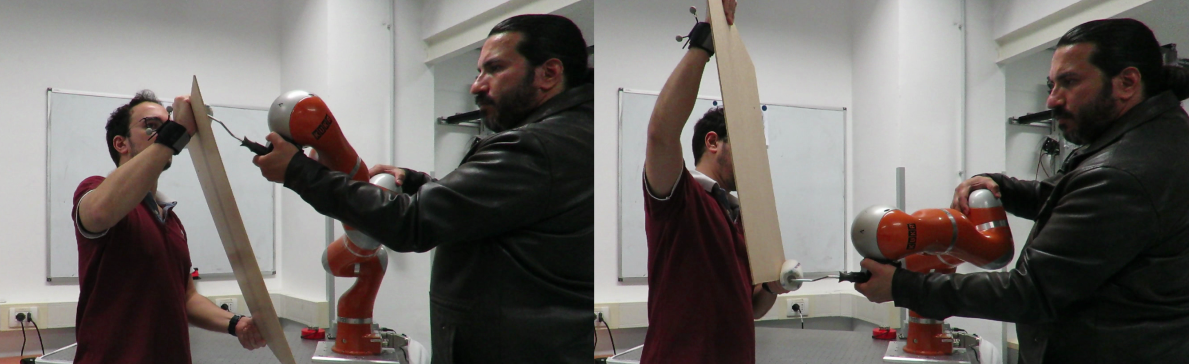}
	\caption{Human operators teach the robot how to apply painting strokes.}
	\label{fig:demostask1-3}
	%	\vspace{-0.5cm}
\end{figure}

\begin{figure}
	\includegraphics[width=\columnwidth]{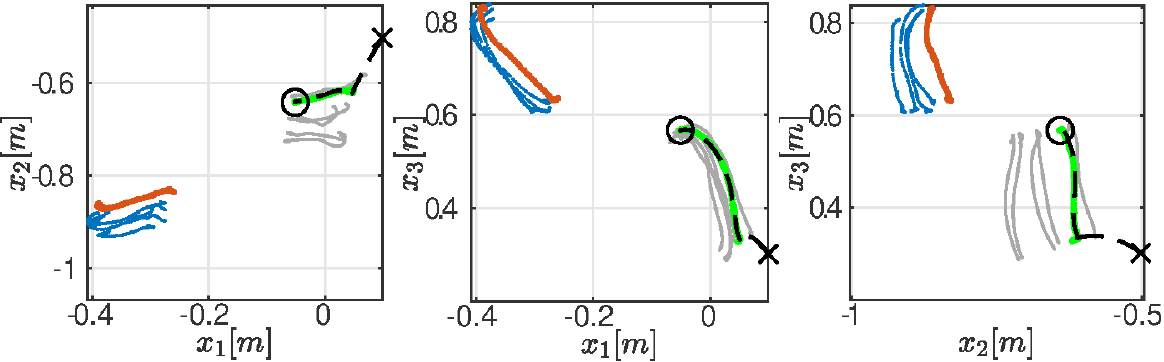}
	\caption{Painting demonstration dataset and reproduction. Training human data (blue), robot data (gray), test human data (orange) and robot desired and observed trajectories (green and black, respectively).}
	\label{fig:paintingdata}
%	\vspace{-0.5cm}
\end{figure}

\begin{figure}
	\includegraphics[width=\columnwidth]{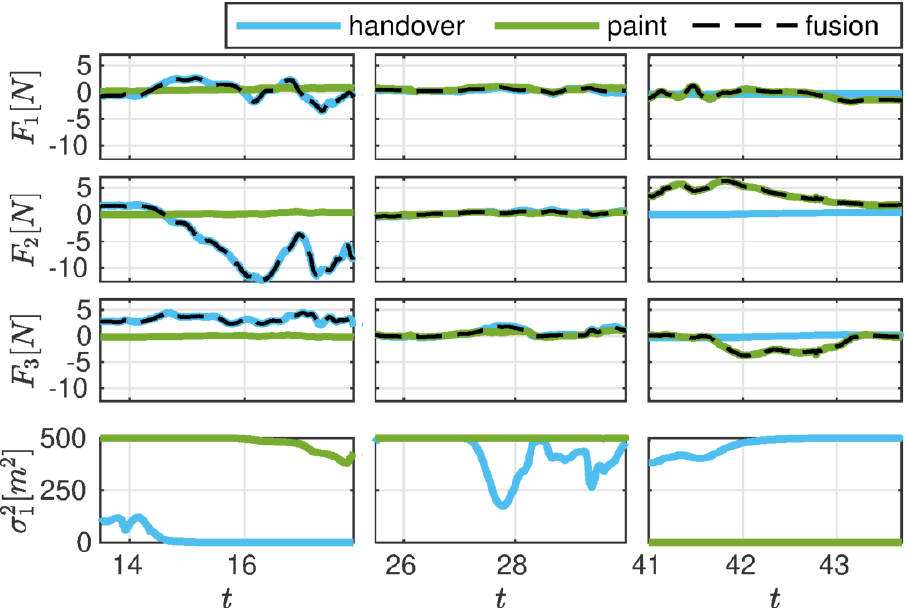}
	\caption{\textbf{Rows 1-3:} Forces generated by each KMP (blue and green) and force used by the robot (black) at three different time intervals of the complete painting task. \textbf{Bottom row:} First entry of the covariance matrix \eqref{eq:KMPcov} predicted by each KMP.}
	\label{fig:forces}
\end{figure}

\begin{figure*}
	\centering
	\begin{subfigure}[b]{0.485\textwidth}
		\includegraphics[width=0.9\textwidth]{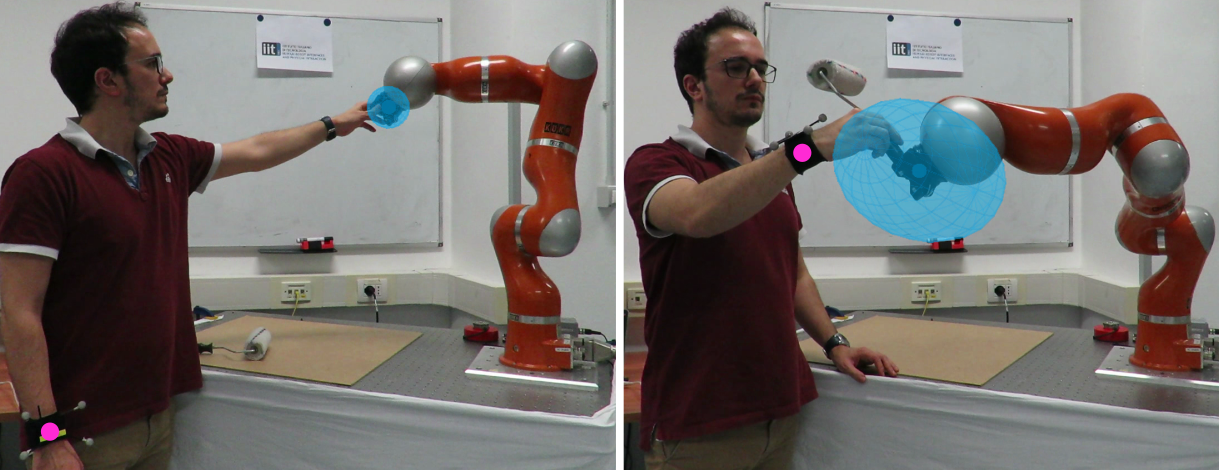}
		\caption{Paint roller handover.}
		\label{fig:repro_handover}		
	\end{subfigure}
	\begin{subfigure}[b]{0.237\textwidth}
		\includegraphics[width=0.9\textwidth]{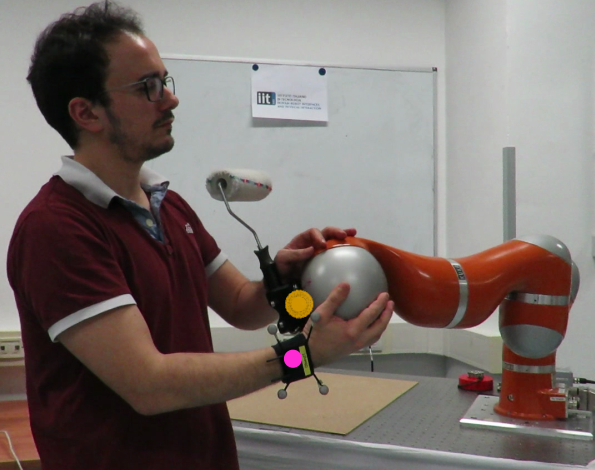}
		\caption{Compliant robot.}
		\label{fig:repro_arbitrary}
	\end{subfigure}
	\begin{subfigure}[b]{0.237\textwidth}
		\includegraphics[width=0.9\textwidth]{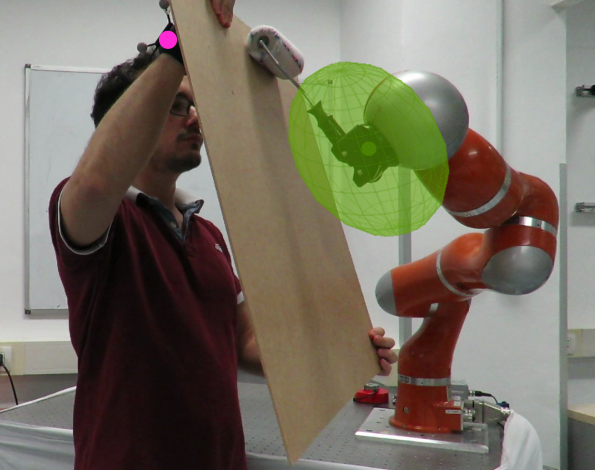}
		\caption{Painting strokes.}
		\label{fig:repro_paint}
	\end{subfigure}
%	\vspace{-0.1cm}
	\caption{\textbf{Fusion of optimal controllers:} snapshots of the complete painting task. The end-effector stiffness is depicted as an ellipsoid at different moments of the task (larger ellipsoids correspond to higher stiffnesses). The position of the human hand (pink) is used to query two KMPs, whose predictions generated both a reference end-effector position and a covariance matrix from which optimal stiffnesses were computed. }
	\label{fig:repros}
%	\vspace{-0.5cm}
\end{figure*}

We provided 5 demonstrations of painting strokes to the robot as shown in Fig. \ref{fig:demostask1-3}. During these demonstrations, the robot learns to map the wooden board motion (as defined by the human hand) to the movements it should perform with the end-effector. We used the same KMP and LQR parameters as in Section \ref{sec:handover}. Figure \ref{fig:paintingdata} shows the data used to train the model, together with one reproduction for a human hand trajectory in the neighborhood of the demonstrations. Note that in this case, the differences in covariance in different parts of the end-effector trajectory are not as accentuated as in the handover task. Nonetheless, optimal gains are computed at every moment according to the observed variability and correlations.

Using two KMPs, each one responsible for a sub-task, we reproduced the complete task where the control commands generated by the two candidate controllers were fused as described in \ref{sec:fusion}. The complete task took about 2 minutes, but here we report about $ 45s $ (the reader is referred to the supplementary video for the full experiment). The first three rows in Fig. \ref{fig:forces} show the forces generated by each candidate controller and the force used by the robot. Due to the unit mass assumption in Section \ref{sec:optimalCtr}, the acceleration commands are equivalent to desired task space forces $ {\mb{F}\in\mathbb{R}^{N_C}} $ which are converted into joint space torques through $ {\mb{\tau} = \mb{J}^\trsp\!\mb{F}}$ \cite{Khatib87} ($ \mb{\tau} $ is a vector of torques and $ \mb{J} $ is the end-effector Jacobian). The bottom row shows the first element of the covariance matrix estimated from \eqref{eq:KMPcov} (the remainder diagonal elements exhibit similar temporal profiles, hence were omitted). The first column ($\approx 14s-18s  $) corresponds to the beginning of the task, where the paint roller is handed over. The force used by the robot closely matches the one generated by the handover KMP, as its predicted variance is significantly lower during the whole task (bottom plot). Notice the increased value of $ \sigma^2_1 $ around $ 14s $ -- it reflects the high variability in the handover demonstrations at the beginning of the task. This value is consistently below the one generated by the painting KMP. If this were not the case, $ N $ could be increased. In the second column of Fig. \ref{fig:forces} ($\approx 26s-30s  $) we can see that both KMP generate high variance. This corresponds to a region in between the two sub-tasks hence none of the two should be activated. The fact that the blue line does not reach $ 500 $ is due to the human moving slightly closer to the region where the handover was demonstrated. Note however that the observed values are consistently higher than those when tasks are activated. Moreover, notice the low forces during this interval -- high covariances yielded minimal control gains resulting in low forces and a compliant robot. Finally in the third column ($\approx 41s-44s  $) we see the task space forces generated during one painting stroke. Notice how, this time, the result from the controller fusion matches the force given by the painting KMP, since the variance for this sub-task is consistently very low (see bottom plot).

Finally, Fig. \ref{fig:repros} shows snapshots of different moments of the reproduction. We draw ellipsoids representing full stiffness matrices at the end-effector in the different task moments. These matrices were estimated from the covariance predictions \eqref{eq:KMPcov} using LQR (Section \ref{sec:optimalCtr}). Figure \ref{fig:repro_handover} shows the two distinct moments of the handover: the beginning, where the end-effector stiffness is low and the user can move the robot around easily, and the end, where the stiffness is high, allowing for the insertion of the paint roller. For an easier visualization we only plotted the stiffness generated by the \textit{handover KMP} (hence the blue color) since the one from the \textit{painting KMP} was negligible in this part of the task (as we saw in Fig. \ref{fig:forces}). Figure \ref{fig:repro_arbitrary} shows a part of the task where none of the two sub-tasks is active. This results in an extremely low stiffness matrix and a fully compliant robot that is safe for the human to interact with and move around in the workspace. In Fig. \ref{fig:repro_paint} the robot performs a painting stroke on the board, driven by the human hand position, with high stiffness since the demonstrations were consistent in this region. The drawn stiffness ellipsoid resulted from the \textit{painting KMP}, since the one from the \textit{handover KMP} was negligible in the vicinity of this sub-task.

\section{Discussion}
\label{sec:discussion}

In the previous section we showed how KMP can be used to estimate full covariance matrices and uncertainty in order to learn optimal and safe controllers, as well as tasks which are comprised of more than one sub-task. One relevant point of discussion is the fact that, unlike \cite{Choi2018}, our approach does not explicitly separate between covariance and uncertainty predictions -- they are both the result of \eqref{eq:KMPcov}. However, we know a priori the form of the uncertainty predictions, as it is defined by the KMP hyperparameters. If desired, one could potentially assign confidence to a prediction as to whether it corresponds to a covariance matrix or an uncertainty matrix. One way to achieve this could be by resorting to heuristics (e.g. the determinant of the prediction, the Frobenius norm of the distance between matrices) to disambiguate between the two possibilities. Alternatively, one could also exploit the freedom given by the hyperparameters in \eqref{eq:uncertainty} to accentuate the difference between the two types of prediction. For instance, by setting very low values for $ \lambda_2 $ (unconstraining covariances), one can increase the uncertainty by several orders of magnitude. The same effect would be achieved by sampling more points into the KMP reference trajectory, thus increasing $ N $. This is helped by the fact that, in practice, there are physical limits to how big variability in the data can be (e.g. joint limits, robot workspace size), hence the uncertainty matrix can often be designed so that it is significantly greater than these.

Finally, in the considered setup the two tasks were performed in different parts of the workspace. This was a design choice, as the 3D human hand position was being used to drive the KMP of each sub-task. However, in practice, our approach can extend to more complicated scenarios. For example, if one was to augment the input vector $\mb{\xi}_\mathcal{I} $ to include other features (e.g. human upper-body configuration, eye gaze), several tasks could potentially overlap in the robot workspace, since there would be more features accounted for by the inputs. This would also lead to a decreased possibility of simultaneously activating undesired sub-tasks.

\section{Conclusions and Future Work}
\label{sec:conclusions}

We proposed an imitation learning approach that takes into account the robot uncertainty about its actions, in addition to the variability and correlations in the data, to estimate optimal controllers from demonstrations. The approach was shown to allow for increased safety, as the robot is compliant when uncertain, and efficiency, with the robot using optimal control gains where demonstrations were given. We also showed that demonstrated tasks can be split into sub-tasks that are activated based on their individual uncertainty levels.

In future work we will study how different kernels can be exploited in our framework. Periodic and composite kernels may allow, respectively, for learning tasks that require rhythmic motions or for the usage of different kernels in different parts of a task. Moreover, we plan to exploit the compliance of the robot when uncertain to provide new demonstrations (of the same or new sub-tasks), in an online learning setting. Finally, it should be noted that the prediction capabilities of KMP are independent of the optimal control framework. Predicting full covariance matrices and uncertainty (as shown here), handling start-/via-/end-points \cite{Huang2019d}, multi-dimensional inputs and orientations \cite{Huang2019} are distinguishable features of KMP that we aim to leverage in other robotics applications.

\bibliographystyle{IEEEtran} % use IEEEtran.bst style
\bibliography{Silverio18.bib}

\end{document}